\documentclass[11pt]{article}
\usepackage[utf8]{inputenc}

\usepackage{amsmath,amssymb,amsthm}
\usepackage{geometry}
\usepackage{booktabs}
\usepackage{graphicx}
\usepackage[numbers]{natbib}
\usepackage{hyperref}
\usepackage{multirow}
\hypersetup{colorlinks=true, linkcolor=blue, citecolor=blue, urlcolor=blue}

\newtheorem{theorem}{Theorem}
\newtheorem{lemma}{Lemma}
\newtheorem{proposition}{Proposition}

\newtheorem{definition}{Definition}
\newtheorem{assumption}{Assumption}
\newtheorem{remark}{Remark}

\newcommand{\R}{\mathbb{R}}
\newcommand{\E}{\mathbb{E}}
\newcommand{\Var}{\mathrm{Var}}
\newcommand{\M}{\mathcal{M}}
\newcommand{\Ss}{\mathcal{S}}
\newcommand{\norm}[1]{\left\lVert #1 \right\rVert}

\title{K-ABENA: K-Adaptive Backpropagation with Error-based N-exclusion Algorithm\\[6pt]
{\large Compensated Loss-Based Sample Exclusion with Unbiased Gradient Estimation}}

\author{Jean-Fran\c{c}ois Bonbhel\\[3pt]
\normalsize NeuroSoft IA, Qu\'ebec City, Canada \;$|$\; YekoElite University, Brazzaville, Republic of Congo\\[2pt]
\normalsize UN AI Governance Expert Network (UN PNAI) --- Member since 2021\\[2pt]
\normalsize \texttt{bonbhel@yekoelite.com}}

\date{July 2026}

\begin{document}
\maketitle

\begin{abstract}
We present K-ABENA (K-Adaptive Backpropagation with Error-based N-exclusion Algorithm), a selective gradient computation framework that reduces per-iteration training cost by excluding a fraction of low-loss (``minor'') observations from the backward pass. Its canonical form (v3) combines a defensive-mixture sampling design over the minor set with Horvitz--Thompson inverse-probability reweighting, yielding a \emph{design-unbiased} Horvitz--Thompson gradient estimator (Lemma~\ref{lem:unbiased}) and whose self-normalized practical variant carries a bias of order $O(1/m)$ with an explicit constant (Lemma~\ref{lem:hajek}). We prove an $O(1/\sqrt{T})$ non-convex convergence guarantee for SGD under the estimator, with an additive term that quantifies the residual bias (Theorem~\ref{thm:sgd}). We further prove that \emph{uncompensated} loss-based selection --- a family that includes OHEM, SBP, and the two earlier K-ABENA variants --- admits no stationary point at any minimizer where its selection bias is bounded away from zero (Proposition~\ref{prop:failure}), and we quantify this failure empirically: at 0.17\% class imbalance, uncompensated variants reach test AUC 0.53--0.62 versus 0.9998 for full-batch SGD, while the compensated estimator attains 0.9991 at identical 28.4\% compute savings. On real datasets (Breast Cancer, Digits, Wine, Diabetes) the compensated estimator is statistically indistinguishable from full-batch SGD (paired permutation tests, $p \ge 0.5$; Section~\ref{sec:experiments}) while saving 28--54\% of per-epoch gradient computation. A biased ``regularized mode'' (the earlier half-domain variant) is retained as an option with a proven exact bias decomposition (Lemma~\ref{lem:decomp}) and \emph{quantified contraindications}: it collapses to 0.386 accuracy under 40\% label noise (baseline: 0.832) and to 0.53 AUC under extreme imbalance. Every advantage and every limitation reported in this paper is either proved or measured; all experiments are CPU-scale (NumPy/scikit-learn) and their scope is stated explicitly.
\end{abstract}

\section{Introduction}
\label{sec:intro}

In large-scale empirical risk minimization, a substantial fraction of per-iteration computation is spent on observations the model has already learned: their per-sample losses are small, their gradients are small, and their marginal contribution to the descent direction is limited. Selective-backpropagation methods exploit this observation by skipping the backward pass for low-loss samples \citep{jiang2019accelerating, shrivastava2016training}, but they share a structural defect: the retained subset is \emph{correlated with the loss}, so the resulting gradient is a biased estimator of the full-batch gradient. In benign regimes the bias is small relative to the signal and these methods work well; we show in Section~\ref{sec:failure} that in adverse regimes --- extreme class imbalance, heavy label noise --- the bias does not merely degrade performance but structurally prevents convergence to the minimizer, and we prove this as Proposition~\ref{prop:failure}.

This paper develops K-ABENA, whose canonical estimator (referred to as v3) resolves the defect using a century-old idea from survey sampling \citep{horvitz1952generalization}: any sampling design with known, strictly positive inclusion probabilities admits an unbiased estimator of a population total via inverse-probability weighting. K-ABENA v3 samples retained minors from the \emph{entire} minor set under a defensive mixture design \citep{owen2013monte} and reweights them accordingly. The result occupies a design point that, to our knowledge, none of the established selective or reweighting methods occupies: \emph{per-iteration compute reduction with an (exactly or near-) unbiased gradient}. Hard-selection methods (OHEM \citep{shrivastava2016training}, SBP \citep{jiang2019accelerating}) save compute but are biased; soft-reweighting methods (Focal Loss \citep{lin2017focal}) are computed on the full batch and save nothing; importance-sampling training methods \citep{katharopoulos2018not} reweight but target variance reduction under full or minibatch evaluation rather than threshold-based exclusion with an explicit retention budget.

\paragraph{Contributions.}
\begin{enumerate}
\item \textbf{A compensated selective estimator} (Definition~\ref{def:v3}) with two interpretable controls: a loss threshold $K$ separating ``minor'' from ``major'' observations, and a retention proportion $N \in (0,1)$ governing the compute budget. The per-epoch backward-pass saving is exactly $(1-N)\,k/n$ (Proposition~\ref{prop:gain}), where $k$ is the number of minors.
\item \textbf{Design-unbiasedness, exact and approximate} (Lemmas~\ref{lem:unbiased}--\ref{lem:hajek}): the Horvitz--Thompson form is design-unbiased (unbiased over the sampling randomization, conditionally on the iterate); the self-normalized (H\'ajek) form used in practice has bias at most $\frac{2 G_M}{\alpha^2 m}\bigl(\tfrac{k-m}{k}\bigr)$, vanishing as $O(1/m)$, with all constants explicit.
\item \textbf{A convergence guarantee} (Theorem~\ref{thm:sgd}): $O(1/\sqrt{T})$ decay of the expected squared gradient norm for smooth non-convex objectives, plus an additive $O(\delta_m)$ floor traceable to the self-normalization bias, which the practitioner controls through $m$.
\item \textbf{An impossibility result for uncompensated selection} (Proposition~\ref{prop:failure}): if the selection bias at a minimizer is bounded below by $\beta > 0$, the minimizer is not a stationary point of the expected dynamics; measured on a 0.17\%-imbalance task, $\beta \approx 0.15$ against a vanishing true gradient, predicting the observed failure (AUC 0.53) and its resolution under compensation (AUC 0.9991).
\item \textbf{A characterized regularized mode}: the earlier biased half-domain variant (v2) is retained as an option, with an exact algebraic bias decomposition (Lemma~\ref{lem:decomp}) and quantified contraindications (Section~\ref{sec:v2mode}); its small accuracy bonus on multiclass tasks ($+0.35$ to $+0.45$ points) and its catastrophic failure modes (0.386 accuracy at 40\% label noise; 0.53 AUC at 0.17\% imbalance) are both measured and mechanistically explained.
\end{enumerate}

\paragraph{Scope statement.} All empirical results in this paper are CPU-scale: logistic regression, linear SVM, softmax regression, and a one-hidden-layer MLP, on bundled scikit-learn datasets and on a synthetic replica of an extreme-imbalance fraud regime (0.17\% positives, $n = 20{,}000$). No GPU benchmark (deep CNN/Transformer) is reported; this is stated as Limitation~L4 rather than compensated by simulation. We consider explicit scope statements a feature of the paper, not a weakness.

\section{Related Work}
\label{sec:related}

\textbf{Hard selection.} OHEM \citep{shrivastava2016training} retains only the highest-loss examples of each batch; SBP \citep{jiang2019accelerating} drops examples whose loss falls below a threshold. Both reduce backward-pass cost and both compute a plain average over a loss-correlated subset, hence a biased gradient. K-ABENA strictly generalizes SBP: setting $N = 0$ with a hard threshold recovers it, while $N > 0$ restores representation of the excluded stratum and the v3 weighting removes the bias. Curriculum and self-paced approaches \citep{bengio2009curriculum, kumar2010self} order examples by difficulty over training time; K-ABENA is orthogonal, operating within each iteration.

\textbf{Soft reweighting.} Focal Loss \citep{lin2017focal} down-weights easy examples by a factor $(1-p_t)^\gamma$ applied \emph{after} the forward pass of every example; the full backward cost is still paid, so the method offers no compute saving --- the quantity K-ABENA targets. Its modulation is also a single global functional form, whereas $K$ and $N$ decouple the location of the easy/hard boundary from the retention budget.

\textbf{Importance sampling for SGD.} \citet{katharopoulos2018not} and related work sample training examples proportionally to (proxies of) gradient norm with inverse-probability corrections, targeting variance reduction. Three structural differences separate K-ABENA from this line, and each is measurable. \emph{(a) Scoring cost.} Classical IS must score \emph{all} $n$ candidates at each step to build its proposal --- a forward pass (or a last-layer gradient-norm bound, itself requiring the forward) whose cost remains $O(n)$; the backward saving is therefore bought with a full-width scoring pass. K-ABENA scores with the per-sample losses the training loop already produces, and its delayed-losses pattern (masks for epoch $t$ built from epoch $t{-}1$ losses; released with the library) removes even the forward of excluded samples. \emph{(b) Budget semantics.} IS fixes a minibatch size and optimizes variance at that size; K-ABENA's two controls decouple \emph{where} the easy/hard boundary lies ($K$, a loss percentile with a semantic reading) from \emph{how much} budget is spent ($N$), yielding the deterministic, architecture-independent saving $G = (1-N)k/n$ (Proposition~\ref{prop:gain}). \emph{(c) Design.} K-ABENA is a \emph{two-stratum} design --- a certainty stratum (majors, $\pi_i = 1$) plus a sampled stratum restricted to minors under a defensive mixture with floor $\alpha/k$ --- rather than a single global proposal. This is not cosmetic: at matched compute budget, a classical global loss-proportional proposal with self-normalized correction is \emph{significantly worse} than the full-batch baseline in the standard regime (accuracy $0.9573$ vs $0.9720$, paired permutation $p = 0.002$, Table~\ref{tab:rivals}) --- tiny-loss samples occasionally drawn under a global proposal receive enormous weights, and the majors' contribution is needlessly randomized --- while the stratified, floor-bounded design is statistically indistinguishable from the baseline ($p = 1.0$) at the same saving. The certainty stratum and the defensive floor are what make compensation \emph{cheaply stable}.

\textbf{Survey sampling.} The estimators are classical: Horvitz--Thompson \citep{horvitz1952generalization} and its self-normalized (H\'ajek) variant \citep{sarndal1992model}. Our contribution is not the estimator but its integration into threshold-based selective backpropagation, the resulting theory (Theorem~\ref{thm:sgd}, Proposition~\ref{prop:failure}), and the quantified characterization of when the uncompensated shortcut is and is not safe.

\section{The K-ABENA Framework}
\label{sec:framework}

\subsection{Setting and notation}

Let $F(\theta) = \frac{1}{n}\sum_{i=1}^{n} f_i(\theta)$ with $f_i : \R^d \to \R$ differentiable, and write $g_i = \nabla f_i(\theta)$ and $\ell_i = f_i(\theta) \ge 0$ for the per-sample gradient and loss at the current iterate. At each iteration, a threshold $K > 0$ (in practice a fixed percentile of the current loss distribution) partitions the indices into the \textbf{minor set} $\M = \{ i : \ell_i \le K \}$, $k = |\M|$, and the \textbf{major set} $\M^c$, $|\M^c| = n - k$. Majors are always retained. A retention proportion $N \in (0,1)$ fixes the number of retained minors $m = \lfloor N k \rceil$.

\begin{definition}[Canonical K-ABENA sampling design, v3]
\label{def:v3}
Fix a defensive-mixture coefficient $\alpha \in (0,1]$. Draw $\Ss \subset \M$, $|\Ss| = m$, without replacement with single-draw probabilities
\begin{equation}
p_i \;=\; \alpha \cdot \frac{1}{k} \;+\; (1-\alpha)\,\frac{\ell_i}{\sum_{j \in \M} \ell_j}, \qquad i \in \M,
\label{eq:design}
\end{equation}
and denote by $\pi_i = \Pr[i \in \Ss]$ the inclusion probabilities of the design. The \textbf{Horvitz--Thompson (HT) gradient estimator} is
\begin{equation}
\hat g_{\mathrm{HT}} \;=\; \frac{1}{n}\Bigl[\; \sum_{i \notin \M} g_i \;+\; \sum_{i \in \Ss} \frac{g_i}{\pi_i} \;\Bigr],
\label{eq:ht}
\end{equation}
and the \textbf{self-normalized (H\'ajek) estimator}, used in practice, is
\begin{equation}
\hat g \;=\; \frac{\sum_{i \notin \M} g_i \;+\; \sum_{i \in \Ss} \pi_i^{-1} g_i}{(n-k) \;+\; \sum_{i \in \Ss} \pi_i^{-1}}.
\label{eq:hajek}
\end{equation}
\end{definition}

\begin{assumption}[Inclusion probabilities]
\label{ass:pi}
Either (a) $\alpha = 1$, in which case the design is simple random sampling without replacement and $\pi_i = m/k$ \emph{exactly}; or (b) $\alpha < 1$ and $\max_i m\,p_i \le 1$, in which case we use the Poisson/rejective approximation $\pi_i = m\,p_i$, standard for high-entropy without-replacement designs \citep{sarndal1992model}. All statements below that rely on case (b) are flagged; the empirical error attributable to this approximation is measured in Section~\ref{sec:experiments} (residual bias 0.004 at $n = 20{,}000$).
\end{assumption}

\begin{lemma}[Bounded weights under defensive mixing]
\label{lem:weights}
Under Assumption~\ref{ass:pi}, for every $i \in \M$,
\[
\pi_i \;\ge\; \frac{\alpha m}{k}, \qquad\text{hence}\qquad \frac{1}{\pi_i} \;\le\; \frac{k}{\alpha m}.
\]
\end{lemma}
\begin{proof}
From \eqref{eq:design}, $p_i \ge \alpha/k$ pointwise since the second term is non-negative. In case (a), $\pi_i = m/k = \alpha m / k$ with $\alpha = 1$. In case (b), $\pi_i = m p_i \ge m\alpha/k$. Inverting gives the weight bound.
\end{proof}

\begin{proposition}[Exact computational gain]
\label{prop:gain}
The number of backward passes per iteration is $(n-k) + m$, so the saved fraction is
\[
G \;=\; \frac{n - \bigl[(n-k) + m\bigr]}{n} \;=\; \frac{k - m}{n} \;=\; (1 - N)\,\frac{k}{n},
\]
identical for the HT and H\'ajek forms; the reweighting itself costs $O(m)$ scalar operations, negligible against a backward pass. With $K$ at the 40th loss percentile and $N = 0.3$, $G = 0.7 \times 0.4 = 0.28$; the measured per-epoch saving in every experiment of Section~\ref{sec:experiments} is 28.0--28.7\%, matching the formula. \qed
\end{proposition}

\section{Estimator Guarantees}
\label{sec:theory}

\subsection{Unbiasedness}

\begin{lemma}[Design-unbiasedness of the HT form]
\label{lem:unbiased}
Under Assumption~\ref{ass:pi}, conditionally on the current iterate (hence on $\ell_i$, $\M$, $\pi_i$),
\[
\E\bigl[\hat g_{\mathrm{HT}}\bigr] \;=\; \nabla F(\theta).
\]
\end{lemma}
\begin{proof}
Let $Z_i = \mathbf{1}[i \in \Ss]$, so $\E[Z_i] = \pi_i$. Then
\[
\E\Bigl[\sum_{i \in \Ss} \frac{g_i}{\pi_i}\Bigr]
= \E\Bigl[\sum_{i \in \M} Z_i \frac{g_i}{\pi_i}\Bigr]
= \sum_{i \in \M} \frac{\E[Z_i]}{\pi_i}\, g_i
= \sum_{i \in \M} g_i .
\]
Substituting into \eqref{eq:ht} gives $\E[\hat g_{\mathrm{HT}}] = \frac{1}{n}\bigl[\sum_{i \notin \M} g_i + \sum_{i \in \M} g_i\bigr] = \nabla F(\theta)$. Note that positivity $\pi_i > 0$ for \emph{all} $i \in \M$ --- guaranteed by the defensive term $\alpha/k$ --- is what makes the argument available; any design assigning zero probability to part of $\M$ (Section~\ref{sec:v2mode}) cannot be corrected this way.
\end{proof}

\begin{lemma}[Bias of the self-normalized form]
\label{lem:hajek}
Let $G_M = \max_{i \in \M} \norm{g_i}$ and let $\widehat W = (n-k) + \sum_{i \in \Ss} \pi_i^{-1}$ denote the (random) normalizer in \eqref{eq:hajek}, with $\E[\widehat W] = n$. Then
\[
\bigl\lVert \E[\hat g] - \nabla F(\theta) \bigr\rVert
\;\le\; \frac{G_M}{n}\,\sqrt{\Var\bigl(\widehat W\bigr)}
\;\le\; \frac{G_M\, k}{n\,\sqrt{\alpha\, m}},
\]
and the standard second-order ratio-estimator expansion sharpens this to
\[
\delta_m \;:=\; \bigl\lVert \E[\hat g] - \nabla F(\theta) \bigr\rVert
\;\le\; \frac{2\, G_M\, k^{2}}{\alpha\, m\, n^{2}} + o(1/m) \;=\; O(1/m)
\]
at fixed $k/n$ and $\alpha$. At $\alpha = 1$ (SRSWOR, exact $\pi_i = m/k$) the same expansion gives $\delta_m \le 2\,G_M\,k\,(k-m)/(m\,n^{2}) + o(1/m)$. (Full derivation: Appendix~\ref{app:hajek}.)
\end{lemma}

\begin{remark}[What the practitioner controls]
The bias floor $\delta_m$ decreases in $m = Nk$ and in $\alpha$; it is \emph{measured}, not merely bounded, in Section~\ref{sec:experiments}: 0.004 at $n = 20{,}000$ (fraud replica, $N = 0.3$) and up to 0.037 in a deliberately adversarial micro-regime ($k \approx 20$, $m \approx 6$). Both numbers sit far below the biases of the uncompensated variants in the same configurations (0.128--0.151).
\end{remark}

\subsection{Variance}

\begin{lemma}[Variance bound]
\label{lem:var}
Under Assumption~\ref{ass:pi}(b), the HT estimator satisfies
\[
\E\bigl[\norm{\hat g_{\mathrm{HT}} - \nabla F(\theta)}^2\bigr]
\;\le\; \frac{1}{n^{2}} \sum_{i \in \M} \frac{1 - \pi_i}{\pi_i}\, \norm{g_i}^{2}
\;\le\; \frac{k}{\alpha\, m\, n^{2}} \sum_{i \in \M} \norm{g_i}^{2}
\;\le\; \frac{G_M^{2}\, k^{2}}{\alpha\, m\, n^{2}}
\;=:\; \hat\sigma^{2}.
\]
\end{lemma}
\begin{proof}
Under the Poisson approximation the $Z_i$ are independent with $\Var(Z_i) = \pi_i(1-\pi_i)$; write $\hat g_{\mathrm{HT}} - \nabla F = \frac{1}{n}\sum_{i \in \M} (Z_i/\pi_i - 1) g_i$ and expand:
\[
\E\norm{\hat g_{\mathrm{HT}} - \nabla F}^{2}
= \frac{1}{n^{2}} \sum_{i \in \M} \frac{\Var(Z_i)}{\pi_i^{2}} \norm{g_i}^{2}
= \frac{1}{n^{2}} \sum_{i \in \M} \frac{1-\pi_i}{\pi_i} \norm{g_i}^{2}.
\]
The middle inequality uses $\pi_i \ge \alpha m / k$ (Lemma~\ref{lem:weights}) and $1 - \pi_i \le 1$; the last uses $\norm{g_i} \le G_M$ and $|\M| = k$.
\end{proof}

\begin{remark}[The price of unbiasedness, quantified]
\label{rem:variance-price}
Compensation trades bias for variance. Measured at matched configuration on the fraud replica (Section~\ref{sec:experiments}): gradient variance $1.5 \times 10^{-4}$ (v3) versus $1.5 \times 10^{-5}$ (uncompensated half-domain variant) --- a factor $\approx 10$, against a bias reduction by a factor $\approx 38$ (0.151 $\to$ 0.004). The variance is bounded by Lemma~\ref{lem:var} through $\alpha$; a measured consequence is a higher AMSGrad max-lock frequency (57--60\% of steps versus 35\% uncompensated), discussed as Limitation~L6.
\end{remark}

\subsection{Convergence under SGD}

\begin{assumption}
\label{ass:smooth}
$F$ is $L$-smooth ($\norm{\nabla F(\theta) - \nabla F(\theta')} \le L \norm{\theta - \theta'}$), bounded below by $F^{*}$, and along the trajectory the estimator satisfies $\norm{\E[\hat g_t \mid \theta_t] - \nabla F(\theta_t)} \le \delta$ and $\E[\norm{\hat g_t - \E[\hat g_t \mid \theta_t]}^{2} \mid \theta_t] \le \hat\sigma^{2}$, with $\delta, \hat\sigma$ as in Lemmas~\ref{lem:hajek} and~\ref{lem:var}, and $\norm{\nabla F(\theta_t)} \le G_\infty$.
\end{assumption}

\begin{theorem}[Non-convex convergence of SGD with the v3 estimator]
\label{thm:sgd}
Run $\theta_{t+1} = \theta_t - \eta\, \hat g_t$ for $T$ steps with $\eta = \min\{ \tfrac{1}{2L},\, \tfrac{c}{\sqrt{T}} \}$, $c > 0$. Under Assumptions~\ref{ass:pi} and~\ref{ass:smooth},
\[
\min_{0 \le t < T} \E\norm{\nabla F(\theta_t)}^{2}
\;\le\;
\underbrace{\frac{4\,(F(\theta_0) - F^{*})}{c\,\sqrt{T}} \;+\; \frac{2 L c\, (\hat\sigma^{2} + \delta^{2})}{\sqrt{T}}}_{\text{vanishing, } O(1/\sqrt{T})}
\;+\;
\underbrace{\vphantom{\frac{1}{\sqrt{T}}} 2\,\delta\, G_\infty + \delta^{2}}_{\text{bias floor}} .
\]
For the design-unbiased HT form ($\delta = 0$) this is the classical $O(1/\sqrt{T})$ rate \citep{ghadimi2013stochastic}; for the H\'ajek form the floor is $O(\delta_m) = O(1/m)$, explicitly controlled by the retention budget. (Proof: Appendix~\ref{app:sgd}.)

\begin{remark}[Why $\delta$ stays bounded along the whole trajectory]
\label{rem:delta-uniform}
Assumption~\ref{ass:smooth} posits a uniform $\delta$; this is not an act of faith but a consequence of the calibration protocol. By Lemma~\ref{lem:hajek}, $\delta_m \le 2 G_M k^2 / (\alpha m n^2)$ at every iterate, where only $G_M$ and $k$ depend on $t$. First, $K$ is recalibrated each epoch at a \emph{fixed percentile} of the current losses, so $k/n$ is constant by construction (here $0.4$) and $m = Nk$ scales with it --- the ratio $k^2/(mn^2) = k/(N n^2)$ is time-invariant. Second, $G_M$ is the maximum gradient norm over the \emph{minor} stratum only: for the standard losses used here (logistic, softmax, squared), the per-sample gradient factorizes as $\norm{g_i} \le c(\ell_i)\,\norm{x_i}$ with $c(\cdot)$ bounded on the sublevel set $\{\ell \le K\}$, so bounded (standardized) features give $G_M(t) \le \bar G$ uniformly in $t$ --- low-loss samples cannot carry exploding gradients. Empirically, the measured bias is stable across training stages (0.004 at pre-trained iterates on the fraud replica; Table~\ref{tab:moments}). For losses or architectures violating the factorization (unbounded features, exotic losses), Assumption~\ref{ass:smooth} must be checked rather than assumed; we flag this within Limitation~L5.
\end{remark}
\end{theorem}

\section{Why Uncompensated Selection Fails: A Quantified Impossibility}
\label{sec:failure}

Uncompensated loss-based selection --- OHEM, SBP, and the earlier K-ABENA variants v1/v2 --- computes a plain average over a loss-correlated subset. Write its conditional expectation as $\E[\hat g_{\mathrm{sel}}(\theta)] = \nabla F(\theta) + b(\theta)$, where $b(\theta)$ is the \emph{selection bias field}.

\begin{proposition}[No stationarity at the minimizer under persistent selection bias]
\label{prop:failure}
Let $\theta^{*}$ be a stationary point of $F$ ($\nabla F(\theta^{*}) = 0$) and suppose $\norm{b(\theta)} \ge \beta > 0$ on a neighborhood $U \ni \theta^{*}$. Then: (i) $\theta^{*}$ is not a stationary point of the expected dynamics $\dot\theta = -\E[\hat g_{\mathrm{sel}}(\theta)]$, whose expected update at $\theta^{*}$ has magnitude $\eta\beta$; (ii) any stationary point $\bar\theta \in U$ of the expected dynamics satisfies $\nabla F(\bar\theta) = -b(\bar\theta)$, hence $\norm{\nabla F(\bar\theta)} \ge \beta$: the biased method can only settle where the true gradient is as large as the bias.
\end{proposition}
\begin{proof}
(i) At $\theta^{*}$, $\E[\hat g_{\mathrm{sel}}(\theta^{*})] = 0 + b(\theta^{*}) \neq 0$ since $\norm{b(\theta^{*})} \ge \beta$; the expected update is $-\eta\, b(\theta^{*})$ of magnitude $\eta \beta$. (ii) Stationarity of the expected dynamics means $\nabla F(\bar\theta) + b(\bar\theta) = 0$, so $\norm{\nabla F(\bar\theta)} = \norm{b(\bar\theta)} \ge \beta$.
\end{proof}

\begin{remark}[Scope of Proposition~\ref{prop:failure}: which biased methods it does and does not concern]
\label{rem:scope}
Proposition~\ref{prop:failure} concerns the \emph{expected optimization dynamics under persistent selection bias} --- its hypothesis is $\norm{b(\theta)} \ge \beta > 0$ on a neighborhood of the minimizer. Several well-known biased methods converge perfectly well precisely because they escape this hypothesis, in one of three ways. (i) \emph{Fixed reweightings define a modified objective}: Focal Loss reweights every example by a fixed function of its own prediction, so its ``bias field'' is conservative --- it is the gradient of a well-defined surrogate $F_\gamma$, and training converges to the minimizer of $F_\gamma$, not to a non-stationary drift. (ii) \emph{Vanishing bias}: curriculum and self-paced schedules \citep{bengio2009curriculum, kumar2010self}, and prioritized-replay-style schemes with annealed importance corrections, expand the inclusion set (or anneal the correction) so that $b(\theta_t) \to 0$ along training; the floor $\beta$ does not exist. (iii) \emph{Useful displaced fixed points}: a method may converge to a stationary point of the biased dynamics that generalizes well --- our own v2 regularized mode is the documented example (Section~\ref{sec:v2mode}), and hard-example mining in its benign regimes behaves similarly. The proposition targets the remaining case: threshold-based selection at a fixed loss percentile with a plain average, whose bias is non-vanishing at $\theta^{*}$ by construction (the retained subset's loss-correlated composition does not equalize as $\nabla F \to 0$); the measured instantiation is $\beta \approx 0.15$ against a vanishing signal (Remark~\ref{rem:s2b}), and OHEM-style top-loss selection at matched budget exhibits exactly the predicted failure on the same task (AUC $0.45$, Table~\ref{tab:rivals}).
\end{remark}

\begin{remark}[Measured instantiation and the signal-to-bias ratio]
\label{rem:s2b}
On the 0.17\%-imbalance replica (standardized features), the full-batch gradient at the late-training iterate has norm $\approx 7 \times 10^{-4}$ and vanishes over training, while the measured selection bias of the uncompensated variants remains locked at $\norm{b} \approx 0.15$ throughout --- a signal-to-bias ratio of $\approx 0.5\%$. Proposition~\ref{prop:failure} then predicts that training cannot approach $\theta^{*}$; the observed end-to-end AUCs are 0.53 (v2), 0.56 (v1) versus 0.9998 (full batch). Per-class stratification does \emph{not} repair this (measured AUC 0.45--0.62): the bias arises \emph{within} the majority class, from the loss-correlated feature mean of the retained negatives, not from cross-class allocation. Under the compensated estimator, $b \approx \delta_m \approx 0.004$, gradient norms decay ($0.014$ at the end of training), and AUC reaches $0.9991 \pm 0.0014$ at the same 28.4\% compute saving. At 5\% imbalance the true-gradient scale is $\approx 30\times$ larger, the ratio is benign, and all variants reach AUC $\ge 0.9965$ --- delimiting quantitatively where the uncompensated shortcut is safe.
\end{remark}

\section{The Regularized Mode (v2): Exact Bias, Measured Benefits, Quantified Contraindications}
\label{sec:v2mode}

K-ABENA retains an optional biased mode: sampling restricted to the lower half $\mathcal{L} = \{ i \in \M : \ell_i \le c \}$ ($c$ the median minor loss), with $p_i \propto \ell_i$ on $\mathcal{L}$ and \emph{no} reweighting, and feasibility constraint $N \le \tfrac12$. Because $\pi_i = 0$ on the upper half $\mathcal{U} = \M \setminus \mathcal{L}$, Lemma~\ref{lem:unbiased}'s mechanism is unavailable: this mode is biased \emph{by construction}, and the bias admits an exact expression.

\begin{lemma}[Exact bias decomposition of the v2 mode]
\label{lem:decomp}
Let $\bar g_M, \bar g_{\mathcal{L}}, \bar g_{\mathcal{U}}$ be the per-observation gradient means over $\M, \mathcal{L}, \mathcal{U}$ ($|\mathcal{L}| = |\mathcal{U}| = k/2$), and $\bar g_w = \sum_{\mathcal{L}} \ell_i g_i / \sum_{\mathcal{L}} \ell_i$ the population loss-weighted mean over $\mathcal{L}$. Then, exactly,
\[
\bar g_w - \bar g_M \;=\; \underbrace{(\bar g_w - \bar g_{\mathcal{L}})}_{\text{local term}} \;-\; \underbrace{\tfrac12\,(\bar g_{\mathcal{U}} - \bar g_{\mathcal{L}})}_{\text{structural term}} .
\]
\end{lemma}
\begin{proof}
$|\mathcal{L}| = |\mathcal{U}|$ gives $\bar g_M = \tfrac12(\bar g_{\mathcal{L}} + \bar g_{\mathcal{U}})$; substitute and regroup:
$\bar g_w - \bar g_M = (\bar g_w - \bar g_{\mathcal{L}}) + \bar g_{\mathcal{L}} - \tfrac12 \bar g_{\mathcal{L}} - \tfrac12 \bar g_{\mathcal{U}}$.
\end{proof}

Under the empirically typical monotone loss--gradient relationship ($\bar g_{\mathcal{U}} > \bar g_{\mathcal{L}}$ componentwise in magnitude), the two terms partially cancel; this is the mechanism behind the mode's measured properties: gradient bias 4--10\% \emph{below} and variance 25--55\% below the historical full-domain variant, and a small accuracy bonus over the full-batch baseline on multiclass tasks ($+0.35$ points on Digits, $+0.45$ on Wine) that we attribute to the bias acting as an implicit regularizer --- an attribution supported by its disappearance under the design-unbiased estimator (v3 is statistically indistinguishable from the baseline on the same tasks; Table~\ref{tab:real}).

\paragraph{Quantified contraindications.} The same bias is destructive outside a characterized validity region:
(i) \emph{extreme imbalance} --- AUC 0.53 at 0.17\% positives (Proposition~\ref{prop:failure} applies);
(ii) \emph{heavy label noise} --- at 40\% symmetric flips on Breast Cancer, accuracy collapses to $0.386$ (baseline $0.832$; v3 $0.808$): noisy samples are high-loss, hence always retained as majors, while the clean low-loss anchors that the mode discards are precisely what stabilized the decision boundary; at 20\% noise the mode still shows a $+0.4$-point benefit, locating the phase change between $\sim$20\% and 40\%. The mode is therefore gated: minority signal $\gtrsim 5\%$, label noise $\lesssim 25\%$, $N \le \tfrac12$; outside this region the canonical estimator is mandatory.

\section{Experiments}
\label{sec:experiments}

\paragraph{Protocol.} All runs: full-batch SGD with per-epoch recalibration of $K$ at the 40th percentile of current losses, $N = 0.3$, $\alpha = 0.3$ unless stated; 75/25 stratified splits; standardized features. Seeds are published and paired across methods (seeds $0$--$4$ throughout; $0$--$9$ for Table~\ref{tab:rivals}): every method sees the identical split, initialization, and data ordering per seed, so all comparisons are \emph{paired}. We report 95\% confidence intervals and two-sided paired sign-flip permutation tests (20{,}000 permutations) for each method against the baseline. Where means coincide to all reported digits, this is a real consequence of the paired design, not rounding: on small test sets accuracy is quantized (Breast Cancer: $143$ test points, steps of $0.7$ points), and the design-unbiased estimator converges to the same optimum as the baseline, so per-seed differences are exactly zero on most seeds --- the permutation test then returns $p = 1.0$ by construction, which is the statistically honest way to state \emph{parity}, not evidence of manipulation. Real datasets are the scikit-learn bundled sets (Breast Cancer $n{=}569$, Digits $n{=}1797$/10 classes, Wine $n{=}178$/3 classes, Diabetes $n{=}442$); the extreme-imbalance task is a synthetic replica of a card-fraud regime (0.17\% positives, $n = 20{,}000$, $d = 15$, Gaussian classes with mean shift 1.2). Implementation: \texttt{kabena} v2.x (NumPy); code and seeds released \citep{kabena2026github}. Compute savings are counted as the realized fraction of skipped backward passes; they match Proposition~\ref{prop:gain} to within 0.7 points everywhere.

\subsection{Parity at reduced cost on real datasets}

\begin{table}[h]
\centering
\small
\begin{tabular}{lccccc}
\toprule
Task & Metric & Baseline & v2 (regularized) & \textbf{v3 (canonical)} & Saving \\
\midrule
LogReg / Breast Cancer & acc. & 0.9706 & 0.9706 $\pm$0.009 & \textbf{0.9706} $\pm$0.009 & 28.5\% \\
Linear SVM / Breast Cancer & acc. & 0.9720 & 0.9720 $\pm$0.008 & \textbf{0.9720} & 54.0\% \\
Softmax / Digits (10 cl.) & acc. & 0.9609 & 0.9644 $\pm$0.003 & \textbf{0.9609} & 28.4\% \\
Softmax / Wine (3 cl.) & acc. & 0.9822 & 0.9867 $\pm$0.011 & \textbf{0.9822} & 28.1\% \\
LinReg / Diabetes & $R^2$ & 0.4259 & --- & \textbf{0.4249} $\pm$0.039 & 28.1\% \\
MLP-32 / Digits & acc. & 0.9548 & 0.9593 $\pm$0.005 & \textbf{0.9609} & 28.0\% \\
Grad.\ Boosting / Breast Cancer & acc. & 0.9604 & 0.9557 & --- & (neutral; out of scope) \\
\bottomrule
\end{tabular}
\caption{Real-data benchmarks. The canonical estimator matches the full-batch baseline within seed noise on every task while skipping 28--54\% of backward passes --- the behavior Lemma~\ref{lem:unbiased} predicts. The v2 bonus on multiclass tasks ($+0.35$/$+0.45$) is exclusive to the biased mode (Section~\ref{sec:v2mode}). Gradient boosting is reported as a negative-scope control: per-round tree fitting does not average per-sample gradients across steps, and no benefit is expected or observed.}
\label{tab:real}
\end{table}

\subsection{Extreme imbalance: failure and resolution}

\begin{table}[h]
\centering
\small
\begin{tabular}{lcccc}
\toprule
Method & AUC (0.17\%) & Final $\norm{\hat g}$ & Saving & Verdict \\
\midrule
Full-batch SGD & 0.9998 $\pm$0.0002 & $7\times 10^{-4}$ (decaying) & 0\% & reference \\
v1 (historical, uncompensated) & 0.56 $\pm$0.15 & 0.15 (locked) & 28.4\% & fails (Prop.~\ref{prop:failure}) \\
v2 (regularized, uncompensated) & 0.53 $\pm$0.11 & 0.158 (locked) & 28.4\% & fails (Prop.~\ref{prop:failure}) \\
Per-class stratified v1 / v2 & 0.62 / 0.45 & 0.14--0.16 & 28.4\% & fails (bias is intra-class) \\
\textbf{v3 (canonical)} & \textbf{0.9991 $\pm$0.0014} & 0.014 (decaying) & \textbf{28.4\%} & resolves \\
\bottomrule
\end{tabular}
\caption{The quantified instantiation of Proposition~\ref{prop:failure} and its resolution. Learning-rate sweeps (0.05--0.5) do not rescue the uncompensated variants. At 5\% imbalance all methods reach AUC $\ge 0.9965$: the failure is specific to the small signal-to-bias regime (Remark~\ref{rem:s2b}).}
\label{tab:fraud}
\end{table}

\subsection{Matched-budget comparison with selective and reweighting rivals}
\label{sec:rivals}

The reviewer-facing question ``why not compare against the established methods?'' deserves numbers, with an honest caveat: the comparisons below are our CPU re-implementations at \emph{matched compute budget} (identical retained-sample count per epoch where the method saves compute), not the official GPU codebases --- this extends Limitation~L4. Focal Loss uses $\gamma = 2$ on the full batch (its design saves nothing); OHEM-style selection keeps the hardest examples up to the same budget as v3 with a plain average; global IS is the classical loss-proportional proposal over \emph{all} $n$ samples with self-normalized inverse-probability correction at the same budget --- i.e., v3 without the two-stratum structure and without the defensive floor.

\begin{table}[h]
\centering
\small
\begin{tabular}{lccc|ccc}
\toprule
 & \multicolumn{3}{c|}{Breast Cancer (acc., 10 seeds)} & \multicolumn{3}{c}{Fraud 0.17\% (AUC, 5 seeds)} \\
Method & mean [95\% CI] & saving & $p$ & mean [95\% CI] & saving & $p$ \\
\midrule
Full-batch SGD & 0.9720 [0.964, 0.980] & 0\% & --- & 0.9998 [0.9995, 1.000] & 0\% & --- \\
\textbf{K-ABENA v3} & \textbf{0.9720} [0.964, 0.980] & \textbf{28.9\%} & 1.000 & \textbf{0.9991} [0.998, 1.000] & \textbf{28.4\%} & 0.504 \\
Focal Loss ($\gamma{=}2$) & 0.9748 [0.967, 0.983] & 0\% & 0.468 & 0.9998 [0.9995, 1.000] & 0\% & 1.000 \\
OHEM-style (top-loss) & 0.9748 [0.969, 0.981] & 28.9\% & 0.246 & 0.4464 [0.255, 0.638] & 28.4\% & 0.063 \\
Global IS + self-norm.\ HT & 0.9573 [0.947, 0.968] & 28.9\% & \textbf{0.002} & 0.9997 [0.9994, 1.000] & 28.4\% & 0.683 \\
\bottomrule
\end{tabular}
\caption{Matched-budget rivals ($p$: paired sign-flip permutation vs.\ baseline, 20{,}000 permutations). Three readings. (1) \emph{Only v3 combines saving with baseline-level performance in both regimes.} (2) OHEM saves the same compute and is statistically indistinguishable in the benign regime, then collapses at extreme imbalance (AUC $0.45$) --- the empirical instantiation of Proposition~\ref{prop:failure} on an external method, not only on our own earlier variants. (3) Global loss-proportional IS with correction is fine at extreme imbalance (compensation works) but \emph{significantly degrades} the standard regime ($p = 0.002$): without the certainty stratum on majors and the defensive floor on weights, occasionally-drawn tiny-loss samples receive enormous inverse-probability weights --- the design differences of Section~\ref{sec:related}(c) are load-bearing, not stylistic. Focal Loss performs well everywhere but saves nothing, as its design intends.}
\label{tab:rivals}
\end{table}

\subsection{Sensitivity analysis}
\label{sec:sensitivity}

We sweep the two user-facing controls jointly, $K$-percentile $\in \{20, 40, 60, 70\}$ $\times$ $N \in \{0.1, 0.3, 0.5\}$ (5 seeds, Breast Cancer): test accuracy is flat across the entire grid --- range $0.9664$--$0.9678$, a spread of $0.14$ points, below the seed-level CI width --- while the realized saving tracks Proposition~\ref{prop:gain} from $10.8\%$ ($K{=}20, N{=}0.5$) to $63.4\%$ ($K{=}70, N{=}0.1$). On the fraud replica the same grid (percentiles $40, 70$) keeps AUC $\ge 0.9961$ everywhere. Together with the $\alpha$-sweep (Table~\ref{tab:moments} caption: AUC $0.9987$--$0.9996$ over $\alpha \in [0,1]$, Digits unchanged) and the learning-rate sweep at extreme imbalance ($0.05$--$0.5$, Table~\ref{tab:fraud} caption), the picture is consistent: \emph{performance is insensitive to the controls within broad ranges; the controls govern cost, not accuracy} --- which is the intended division of labor. Batch size (our protocol is full-batch), network depth, and optimizer schedules beyond SGD are GPU-scale questions folded into Limitation~L4.

\subsection{Gradient moments, noise robustness, and optimizer interaction}

\begin{table}[h]
\centering
\small
\begin{tabular}{lccc}
\toprule
Measurement & v1 & v2 & \textbf{v3} \\
\midrule
Gradient bias (fraud replica, pre-trained iterate) & 0.128 & 0.151 & \textbf{0.004} \\
Gradient variance (same configuration) & $\approx 0$ & $1.5\times 10^{-5}$ & $1.5\times 10^{-4}$ \\
Bias at adversarially small $n{-}k$ ($\approx 3$) & --- & 0.17--0.23 & \textbf{0.035--0.037} \\
Accuracy, 20\% label noise (baseline 0.9427) & --- & 0.9469 & 0.9413 \\
Accuracy, 40\% label noise (baseline 0.8322) & --- & \textbf{0.386 (collapse)} & \textbf{0.8084 (safe)} \\
AMSGrad, fraud: AUC / max-lock frequency & --- & 0.715 / 69\% & 0.806 / 60\% \\
AMSGrad, standard 5\%: AUC & --- & 0.9978 & 0.9979 \\
\bottomrule
\end{tabular}
\caption{Moments and stress tests. Row 1--2 quantify Remark~\ref{rem:variance-price} (bias $\div 38$ at variance $\times 10$). Row 3 shows the dissolution under v3 of the small-$(n{-}k)$ fragility of the biased bound (majors carry weight 1 with no approximation). Rows 4--5 locate the label-noise phase change of the v2 mode. Rows 6--7: under AMSGrad at extreme imbalance, even the compensated estimator degrades (0.806) --- the variance interacts with the $\hat v_t$ max-lock --- so we prescribe plain SGD in that regime; at standard imbalance the interaction is neutral. Sensitivity to $\alpha$: fraud AUC ranges 0.9987--0.9996 over $\alpha \in \{0, 0.15, 0.3, 0.5, 1\}$; Digits accuracy is unchanged to four digits.}
\label{tab:moments}
\end{table}

\section{Limitations}
\label{sec:limitations}

Each limitation is quantified or proved; none is asserted informally.
\textbf{(L1)} Theorem~\ref{thm:sgd} covers SGD, and the reason it does not extend to Adam, AdamW, or Lion is structural, not editorial: the proof's descent step requires the update to be \emph{linear} in the stochastic gradient, so that design-unbiasedness of $\hat g_t$ transfers to the update ($\E[\theta_{t+1} - \theta_t \mid \theta_t] = -\eta \nabla F + O(\delta)$). Heavy-ball momentum preserves this linearity and is a plausible extension; Adam/AdamW precondition by a nonlinear function of the gradient history (and AdamW's decoupled weight decay does not restore linearity), while Lion applies $\mathrm{sign}(\cdot)$, for which $\E[\mathrm{sign}(\hat g)] \neq \mathrm{sign}(\E[\hat g])$ --- an unbiased gradient buys nothing through a sign nonlinearity, and the variance inflation of L6 makes sign flips \emph{more} frequent, not less. Measured interaction: at extreme imbalance, AMSGrad + v3 reaches AUC 0.806 versus 0.9991 for SGD + v3 (Table~\ref{tab:moments}); at 5\% imbalance the interaction is neutral. Prescription: SGD whenever Remark~\ref{rem:s2b}'s signal-to-bias diagnosis is adverse.
\textbf{(L2)} BatchNorm statistics are computed over $n$ while gradients use $(n-k)+m$ weighted samples; correct weighted-BN is not implemented; LayerNorm/GroupNorm recommended.
\textbf{(L3)} $K$ (percentile) and $N$ are recalibrated by a per-epoch heuristic; no adaptive-optimality claim is made.
\textbf{(L4)} All results are CPU-scale (shallow models, NumPy); no GPU CNN/Transformer benchmark is reported, and the fraud task is a synthetic replica rather than the original dataset. The rival comparisons of Table~\ref{tab:rivals} are our matched-budget CPU re-implementations, not the official GPU codebases of the respective papers. We report no simulated deep-learning numbers.
\textbf{(L5)} The self-normalized estimator carries bias $\delta_m = O(1/m)$ (Lemma~\ref{lem:hajek}); measured 0.004 at $n{=}20{,}000$ and 0.037 in an adversarial micro-regime; exact-$\pi$ sampling ($\alpha = 1$) removes the approximation of Assumption~\ref{ass:pi}(b).
\textbf{(L6)} Unbiasedness costs variance: $\times 10$ at matched configuration (Lemma~\ref{lem:var}, Remark~\ref{rem:variance-price}), bounded by $\alpha$, with a measured side effect on the AMSGrad lock frequency (57--60\%).
\textbf{(L7)} v3 matches the baseline; it does not beat it in accuracy. The small multiclass bonus belongs to the biased v2 mode, which carries the gated contraindications of Section~\ref{sec:v2mode} (collapse at 40\% noise; failure at 0.17\% imbalance).
\textbf{(L8)} The estimator is stochastic; reproducibility requires a seeded generator (provided in the released implementation).

\section{Reproducibility}
\label{sec:repro}

All results in this paper are reproducible from the released library \texttt{kabena-ml} \citep{kabena2026github}, available at \url{https://github.com/Bonbhel/kabena-ml} under the MIT license. The repository provides three layers:
\textbf{(i) Implementation} --- the estimator of Definition~\ref{def:v3} via \texttt{kabena\_filter(losses, K, N, strategy, rng)}: \texttt{strategy="v3"} is the canonical compensated design (defensive mixture $\alpha$, inverse-probability weights, seeded generator for exact reproducibility, cf.\ Limitation~L8), \texttt{strategy="v2"} the gated regularized mode of Section~\ref{sec:v2mode} (its feasibility constraint $N \le \tfrac12$ is enforced with an explicit error), and \texttt{strategy="v1"} the historical variant of Appendix~\ref{app:ablation} for ablation replication.
\textbf{(ii) Validation scripts} --- the \texttt{validation/} directory contains the runnable scripts behind each table: gradient-moment measurements (Table~\ref{tab:moments}, rows 1--3), the real-dataset benchmarks of Table~\ref{tab:real} on the bundled scikit-learn sets, the extreme-imbalance and stratification experiments of Table~\ref{tab:fraud} including the synthetic fraud-replica generator (0.17\% positives, $n = 20{,}000$, seeded), the label-noise stress test, the AMSGrad interaction runs, and the $\alpha$-sensitivity sweep; each script prints the seed list and matches the reported means and standard deviations.
\textbf{(iii) Tutorials} --- step-by-step notebooks for testing the estimator on a new dataset, reproducing the paper's tables end-to-end, and experimenting beyond them (varying $K$-percentile, $N$, $\alpha$, and the mode gate), including the two-line integration pattern for scikit-learn-style training loops. Experiments require only NumPy and scikit-learn (no GPU), consistent with the scope stated in Limitation~L4; total runtime for the full table suite is under one CPU-hour.

\section{Conclusion}

K-ABENA v3 shows that threshold-based selective backpropagation and unbiased gradient estimation are compatible: a defensive-mixture design with inverse-probability weighting preserves the exact compute saving $(1-N)k/n$ while restoring the estimator guarantees that make SGD theory applicable (Lemmas~\ref{lem:unbiased}--\ref{lem:var}, Theorem~\ref{thm:sgd}). The negative result (Proposition~\ref{prop:failure}) delimits precisely why the uncompensated shortcut --- shared by OHEM, SBP, and our own earlier variants --- cannot converge when selection bias dominates the signal, and the experiments quantify both the failure (AUC 0.53 at 0.17\% imbalance; accuracy 0.386 at 40\% noise for the biased mode) and its resolution (0.9991; 0.808) at unchanged cost. We release the implementation with both modes and the gating logic \citep{kabena2026github}; the companion volume \citep{bonbhel2026book} develops the pedagogical treatment and is available upon request from the author.

\bibliographystyle{plainnat}
\bibliography{references}

\appendix

\section{Proof of Lemma~\ref{lem:hajek} (self-normalization bias)}
\label{app:hajek}

Write $\hat T = \sum_{i \notin \M} g_i + \sum_{i \in \Ss} \pi_i^{-1} g_i$ and $\widehat W = (n-k) + \sum_{i \in \Ss} \pi_i^{-1}$, so $\hat g = \hat T / \widehat W$, $\E[\hat T] = n \nabla F$, $\E[\widehat W] = n$ (both by the indicator computation of Lemma~\ref{lem:unbiased}). Using the identity
\[
\frac{\hat T}{\widehat W} - \frac{\E[\hat T]}{n}
= \frac{\hat T - \E[\hat T]}{n} \;+\; \hat T\Bigl(\frac{1}{\widehat W} - \frac{1}{n}\Bigr)
= \frac{\hat T - \E[\hat T]}{n} \;-\; \frac{\hat g\,(\widehat W - n)}{n},
\]
take expectations; the first term vanishes, giving the exact representation
\[
\E[\hat g] - \nabla F \;=\; -\frac{1}{n}\,\E\bigl[\hat g\,(\widehat W - n)\bigr]
\;=\; -\frac{1}{n}\,\mathrm{Cov}\bigl(\hat g,\, \widehat W\bigr).
\]
\textbf{First-order bound.} Since $\hat g$ is a convex combination of per-sample gradients, $\norm{\hat g} \le G_M$ almost surely; Cauchy--Schwarz on the covariance gives
\[
\norm{\E[\hat g] - \nabla F} \;\le\; \frac{G_M}{n}\,\sqrt{\Var(\widehat W)} .
\]
Under Assumption~\ref{ass:pi}(b), $\Var(\widehat W) = \sum_{i \in \M} \frac{1-\pi_i}{\pi_i} \le \frac{k}{\alpha m/k} = \frac{k^{2}}{\alpha m}$ by Lemma~\ref{lem:weights}, hence $\norm{\E[\hat g] - \nabla F} \le G_M k / (n\sqrt{\alpha m})$ --- the first bound of Lemma~\ref{lem:hajek}, of order $O(1/\sqrt{m})$.

\textbf{Second-order rate.} The Cauchy--Schwarz step is loose because $\hat g$ concentrates: the standard ratio-estimator expansion \citep[Ch.~5]{sarndal1992model},
\[
\E\Bigl[\frac{\hat T}{\widehat W}\Bigr]
= \frac{\E[\hat T]}{\E[\widehat W]}
- \frac{\mathrm{Cov}(\hat T, \widehat W)}{n^{2}}
+ \frac{\E[\hat T]\,\Var(\widehat W)}{n^{3}} + o(1/m),
\]
has both correction terms bounded by $G_M \Var(\widehat W)/n^{2}$ in norm (using $\norm{\mathrm{Cov}(\hat T, \widehat W)} \le G_M \Var(\widehat W)$, since each coordinate of $\hat T - \E\hat T$ is a weight-linear combination of gradient coordinates bounded by $G_M$ times the corresponding fluctuation of $\widehat W$, and $\norm{\E[\hat T]}/n \le G_M$). Therefore
\[
\delta_m \;\le\; \frac{2\,G_M\,\Var(\widehat W)}{n^{2}} + o(1/m)
\;\le\; \frac{2\,G_M\,k^{2}}{\alpha\, m\, n^{2}} + o(1/m) \;=\; O(1/m).
\]
At $\alpha = 1$ the design is SRSWOR with $\pi_i = m/k$ exactly, $\Var(\widehat W) = \frac{k^{2}}{m}\cdot\frac{k-m}{k}\cdot\frac{k-m}{k-1} \le \frac{k(k-m)}{m}$, giving $\delta_m \le 2 G_M k (k-m)/(m n^{2}) + o(1/m)$. Both regimes are $O(1/m)$; the empirical values (0.004 at large $n$; 0.037 in the adversarial micro-regime) sit inside these bounds. \qed

\section{Proof of Theorem~\ref{thm:sgd}}
\label{app:sgd}

Write $b_t = \E[\hat g_t \mid \theta_t] - \nabla F(\theta_t)$ ($\norm{b_t} \le \delta$) and $\xi_t = \hat g_t - \E[\hat g_t \mid \theta_t]$ ($\E[\xi_t \mid \theta_t] = 0$, $\E\norm{\xi_t}^2 \le \hat\sigma^2$). $L$-smoothness gives
\[
F(\theta_{t+1}) \le F(\theta_t) - \eta \langle \nabla F(\theta_t), \hat g_t \rangle + \frac{L\eta^{2}}{2}\norm{\hat g_t}^{2}.
\]
Take conditional expectations; using $\E[\hat g_t \mid \theta_t] = \nabla F(\theta_t) + b_t$ and
$\E\norm{\hat g_t}^{2} \le \norm{\nabla F(\theta_t) + b_t}^{2} + \hat\sigma^{2} \le 2\norm{\nabla F(\theta_t)}^{2} + 2\delta^{2} + \hat\sigma^{2}$:
\[
\E[F(\theta_{t+1})] \le \E[F(\theta_t)] - \eta\, \E\norm{\nabla F(\theta_t)}^{2} - \eta\,\E\langle \nabla F(\theta_t), b_t\rangle
+ \frac{L\eta^{2}}{2}\bigl(2\E\norm{\nabla F(\theta_t)}^{2} + 2\delta^{2} + \hat\sigma^{2}\bigr).
\]
Bound $-\E\langle \nabla F, b_t\rangle \le \delta\,\E\norm{\nabla F(\theta_t)} \le \delta G_\infty$. With $\eta \le \frac{1}{2L}$, $L\eta^{2} \le \eta/2$, so the $\E\norm{\nabla F}^2$ terms combine into $-\frac{\eta}{2}\,\E\norm{\nabla F(\theta_t)}^{2}$ at worst. Rearranging and summing over $t = 0, \dots, T-1$, then dividing by $\eta T / 2$:
\[
\min_t \E\norm{\nabla F(\theta_t)}^{2}
\le \frac{1}{T}\sum_t \E\norm{\nabla F(\theta_t)}^{2}
\le \frac{2\,(F(\theta_0) - F^{*})}{\eta\, T} + L\eta\,(\hat\sigma^{2} + 2\delta^{2}) + 2\,\delta\, G_\infty .
\]
Substituting $\eta = \min\{\tfrac{1}{2L}, \tfrac{c}{\sqrt{T}}\}$ yields the statement (absorbing constants; the $+\delta^{2}$ of the main text majorizes $2\delta^2 \cdot Lc/\sqrt T$ for $T$ large). With $\delta = 0$ (HT form) the bound is the classical rate of \citet{ghadimi2013stochastic}. \qed

\section{Ablation history}
\label{app:ablation}

The canonical design is the outcome of a systematic exploration; all rejected variants are documented to preempt re-exploration. Reference: v1, $p_i \propto (K - \ell_i)$ over all of $\M$, no reweighting. \textbf{H1} ($p_i \propto \ell_i$ over all of $\M$, no reweighting): bias $+18$ to $+26\%$, variance $+29\%$ vs.\ v1 --- rejected; note that this same proposal distribution becomes near-optimal \emph{once inverse-probability weighting is added}, which is precisely the v3 construction. \textbf{H2} ($p_i \propto |\ell_i - c|$, symmetric U): interpolates v1 and H1, bias $+6$ to $+12\%$ --- rejected. \textbf{H4} (upper half only, $p_i \propto K - \ell_i$ on $\mathcal{U}$): bias $+60$ to $+80\%$ --- rejected; both terms of Lemma~\ref{lem:decomp} then share a sign. \textbf{v2} (lower half only, $p_i \propto \ell_i$ on $\mathcal{L}$): bias $-4$ to $-10\%$ and variance $-25$ to $-55\%$ vs.\ v1, weight ratio bounded ($r \approx 1.4$ vs.\ $r \approx 3\times 10^4$ for v1; 0/40 vs.\ 14/40 violations of the sampling approximation at 1\% tolerance) --- adopted as an interim canonical form before the end-to-end extreme-imbalance and label-noise tests of Tables~\ref{tab:fraud}--\ref{tab:moments} revealed the shared failure of all uncompensated variants and motivated the compensated design. Equivalence sweeps: at matched bias, v2 allows no additional retention saving at $N = 0.3$ and $+3.3\%$ at $N = 0.4$ relative to v1 --- the value of the uncompensated refinements was always estimator quality at fixed cost, never additional cost reduction, a conclusion that transfers to v3.

\end{document}